\documentclass[conference]{IEEEtran}
\IEEEoverridecommandlockouts

\usepackage{cite}
\usepackage{amsmath,amssymb,amsfonts}
\usepackage{algorithmic}
\usepackage{graphicx}
\usepackage{textcomp}
\usepackage{xcolor}
\usepackage{multicol}
\usepackage{subfig}
\usepackage{booktabs}
\usepackage[draft]{hyperref}
\usepackage{float}
\usepackage[ruled, vlined, noend, boxed]{algorithm2e}
\usepackage{amsthm}
\usepackage{tcolorbox}
\usepackage{multirow}
\usepackage{makecell}

\newtheorem{theorem}{Theorem}
\newtheorem{lemma}{Lemma}[theorem]
\newtheorem{claim}{Claim}[theorem]

\newtheorem{insight}{Finding}

\begin{document}

\title{\textsc{Director}: Accelerating Distributed MoE Serving via Online Proactive Expert Placement}

\author{

\IEEEauthorblockN{
Qianli Liu$^{1}$, 
Kaibin Guo$^{2}$,
Zicong Hong$^{1}$, 
Peng Li$^{3}$, 
Fahao Chen$^{4}$, 
Haodong Wang$^{1}$,
Jian Lin$^{1}$,
and 
Song Guo$^{1}$}
\IEEEauthorblockA{$^1$Department of Computer Science and Engineering, The Hong Kong University of Science and Technology, Hong Kong \\$^2$School of Software Engineering, Sun Yat-Sen University, China\\$^3$School of Cyber Science and Engineering, Xi’an Jiaotong University, China\\$^4$School of Artificial Intelligence, Shandong University, China
\\qianli.liu@connect.ust.hk, guokb@mail2.sysu.edu.cn, ziconghong@gmail.com, pengli@xjtu.edu.cn, \\chenfh@ieee.org, hwanghb@connect.ust.hk, jlindc@connect.ust.hk, songguo@cse.ust.hk}

\thanks{
This research was supported by fundings from the Hong Kong RGC General Research Fund (152169/22E, 152228/23E, 162161/24E, 162116/25E), Research Impact Fund (No. R5060-19, No. R5011-23), Collaborative Research Fund (No. C1042-23GF), NSFC/RGC Collaborative Research Scheme (Grant No. 62461160332 \& CRS\_HKUST602/24), Areas of Excellence Scheme (AoE/E-601/22-R), National Natural Science Foundation of China (No. 62471383), and the InnoHK (HKGAI). Corresponding authors: Zicong Hong, Song Guo. 
}

}

\maketitle

\begin{abstract}
Expert parallelism has become the prevailing paradigm to serve Mixture-of-Experts (MoE) models. 
Its efficiency depends on the communication and computation latencies of the GPUs, which are linked to the placement of experts in the GPUs. 
Existing works for optimizing expert placement focus on leveraging past requests' expert activation patterns. 
However, they demonstrate deficiencies facing diverse and rapidly changing request patterns, calling for an online, proactive approach. 
Implementing such  an approach requires addressing several challenges: the uncertainty associated with incoming requests' expert activation, the cost of expert migration, and the NP-hard complexity in optimization.
Therefore, we present \textsc{Director}, a new distributed MoE serving system that minimizes end‑to‑end latency via prediction‑driven, online expert placement. 
\textsc{Director} uses either a lightweight cascaded predictor or a low‑bit quantized replica for expert activation patterns of incoming requests. 
An online migration module then enacts the changes with near‑zero downtime by executing migrations in compute‑bound phases, keeping disruption bounded. 
At its core, a relaxation‑based expert placement optimizer operates under capacity constraints, runs in polynomial time, and achieves a $(1+\epsilon)$ approximation ratio.
Finally, we implement a prototype and demonstrate, through extensive experiments, a reduction in end-to-end latency of $11\sim55\%$ for popular MoE models (e.g., Mistral, DeepSeek and Qwen) compared to existing work.
\end{abstract}





\section{Introduction}
\label{sec:introduction}

Mixture-of-Experts (MoE) models~\cite{llama4_meta_2025,qwen1.5,qwen2,qwen3technicalreport,deepseek_moe,deepseek_v2,deepseek_r1} have become a popular architecture for scaling large language models (LLMs) to trillions of parameters while keeping the computational cost proportional.
Unlike dense models, which use all parameters, MoE models dynamically route input tokens to a subset of feed-forward networks (FFNs), known as \emph{experts}.
This selective activation allows the required FLOPs to scale sub-linearly as the model size increases~\cite{shazeer2017outrageously,fedus2022switch,zhou2024training}.

The increasing parameter count of MoE models requires expert parallelism for distributed MoE serving due to the significant memory footprint~\cite{lepikhin2020gshard}. In this paradigm, the experts are distributed across a cluster of GPUs, while non-expert parameters, such as self-attention blocks, are replicated on every GPU. A forward pass through each MoE layer requires a pair of all-to-all communication steps: one to dispatch tokens to their selected experts, and another to aggregate the results. This paradigm effectively utilises the sparse, conditional computation of MoE, establishing it as a widely adopted standard for serving large-scale MoE models.

Despite the advantages of expert parallelism, two performance bottlenecks have been identified: high latency during frequent all-to-all communication and an imbalanced computational load across GPUs.
The communication bottleneck arises because the hidden state of each token must be dispatched to the selected experts during all-to-all operations for each MoE layer. These experts are often dispersed across the GPUs, and the hidden state is large. 
Furthermore, the imbalanced load is caused by a shift in distribution between the inference workloads and the model's training data.
This non-uniform selection results in stragglers. The GPUs hosting these over-selected experts determine the overall latency, as the aggregation cannot proceed until the last expert has finalised its computation.

\begin{figure}[t]
	\centering
        \includegraphics[width=0.75\linewidth]{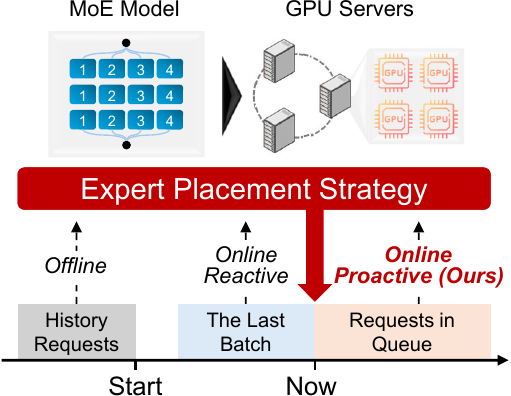}
	\caption{A comparison of offline expert placement, online reactive expert placement, and online proactive expert placement.
    }
    \label{fig:intro}
    \vspace{-0.5cm}
\end{figure}

To accelerate the distributed MoE serving, as shown in \autoref{fig:intro}, recent works can be classified into two strategies: \emph{offline expert placement} and \emph{online reactive expert placement}.
The offline works compute a static expert placement before serving begins. By analyzing historical routing patterns and inter-layer expert affinities, they co-locate frequently co-activated experts to minimize communication and duplicate frequent experts to balance the load~\cite{lepikhin2020gshard,eplb,go2025moetuneroptimizedmixtureexpert}. 
However, these works rely on pre-profiled data, which renders their static placement suboptimal when the inference workload exhibits a shift in data distribution, resulting in performance degradation.
In contrast, the online works adjust to the workload. They respond to dynamic workload by replicating overloaded experts~\cite{he2022fastermoe,flexmoe,hwang2023tutel} or choosing whether to move experts or requests~\cite{zhai2023smartmoe,liu2023janus}. 
However, they often adopt an online reactive design which optimizes for the most recent batch, so decisions lag behind the workload; under rapidly shifting inputs, this lag yields persistent misalignment and poor performance (\autoref{sec:motivation}).

Therefore, we propose a new paradigm: \emph{online proactive expert placement} for distributed MoE serving. This would analyze the requests in the queue and optimize the placement of experts on GPUs before processing them.
This anticipatory placement would co-locate experts likely to be activated together on the same GPU for the requests to be processed, thereby reducing the need for expensive inter-GPU communication. Furthermore, it ensures that the computational workload is distributed evenly across all GPUs, thereby maintaining system-wide load balancing and preventing bottlenecks caused by overloaded GPUs.

However, realizing this paradigm faces three key challenges: 
\textbf{1) Routing Path Uncertainty.}
In MoE, the gate at layer $l$ is evaluated on the representation produced by the experts chosen at layer $l\!-\!1$, routing is data‑dependent and sequentially coupled across layers. 
Hence, a token’s layerwise routing cannot be determined a priori; it is revealed only during the forward pass and must be estimated by prediction.
\textbf{2) Cost of Adaptivity.} 
Migrating experts at run time competes directly with computation: while an expert’s parameters are in transit, that expert cannot do computations, and the data transfer itself introduces communication overhead.  
An effective system must therefore balance migration and computation, scheduling transfers to limit the downtime and ensuring that any reconfiguration delivers a performance gain.
\textbf{3) Complexity of Expert Placement.}
The search space for an optimal expert placement is immense and relying on brute-force search or exact solvers to find the optimal placement is infeasible for real-time decision-making. 
For example, DeepSeekMoE 16B contains 27 MoE layers, each with 64 experts. When deployed across 4 GPUs in expert parallelism. 
In each layer, the number of ways to assign 16 experts to the first GPU is given by the combination $C_{16}^{64}$. 
Given that the model has 27 layers that can be configured independently, the total size of the placement search space approaches $(C_{16}^{64})^{27}$. 
Formally, the task of finding an optimal placement is an NP-hard problem~\cite{savelsbergh1997branch}.

To address these challenges, this paper introduces \textsc{Director}, a distributed MoE serving system that performs prediction-driven, online expert placement reconfiguration. 
\textsc{Director} uses either a cascaded predictor or a quantized MoE replica to accurately predict routing, thereby enabling the proactive updating of expert placements. 
A computation-overlapped migration module then enacts changes with near-zero downtime by executing transfers only in compute-bound phases, keeping disruption bounded.
We summarize our contributions as follows.

\begin{itemize}
    \item We design an online mechanism that couples routing prediction with execution and performs computation‑overlapped live migration with near‑zero downtime, enabling the system to anticipate routing and adjust placements with bounded impact.
    \item We develop a polynomial-time scheduling algorithm that finds an expert placement with a provable performance bound. The algorithm efficiently balances communication costs and GPU load to reduce overall latency.
    \item We implement and evaluate a prototype of \textsc{Director}. Through extensive experiments on diverse workloads and hardware configurations, we demonstrate a significant end-to-end latency reduction of $11\sim55\%$ compared to existing MoE serving systems.
\end{itemize}

\section{Background \& Related Work}
\label{sec:background}
\subsection{Mixture of Experts}

MoE architecture replaces the standard FFN with an MoE layer in the transformer~\cite{shazeer2017_MoE}. 
Each MoE layer consists of a set of FFN sub-networks known as \emph{experts} and a learnable \emph{gating network} that routes input tokens to them.
For each incoming token, the gating network computes relevance scores and selects a small subset of experts for processing, commonly top-k routing. 
After the selected experts have processed the token, the outputs of these experts are aggregated, typically through a weighted summation based on the scores assigned by the gating network.
This sparse activation increases the model's total parameter count without proportionally increasing the computation requirement for a single forward pass.

\subsection{Expert Parallelism}

\label{expert_parallelism}

In expert parallelism, the experts are distributed to different GPUs, with parameters such as the self-attention block, layer normalization and gating network being replicated on each GPU~\cite{lepikhin2020gshard,shazeer2017outrageously}. 
During a forward pass, each GPU first applies these replicated components to its local tokens. The gate then selects the experts for each token, after which the tokens are packaged and sent to the GPUs hosting their targeted experts via all-to-all communication. Once the expert computations have finished, a second all-to-all communication process returns the token outputs to their original GPUs. This enables the model to proceed to the next layer.

\subsection{Optimization for Expert Parallelism}

Various optimisations have been introduced to improve the efficiency of distributed MoE training and inference.

\paragraph{Model Architecture \& Routing Rule}
Some works modify the model's structure to increase the likelihood of a token being processed locally. For instance, expert pruning decreases the number of experts, thereby reducing inter-GPU communication~\cite{lu-etal-2024-experts, liu2024efficient}. Other works optimize the gate's selection in the forward pass. DeepSeek~\cite{deepseek_v2,deepseek_r1}, FasterMoE~\cite{he2022fastermoe} and TA-MoE~\cite{tamoe} enforce a constraint that limits the number of experts that can be selected from other GPUs. 
As our work focuses on optimizing the placement of experts at a system level, these algorithm-level optimizations are orthogonal.

\paragraph{Communication-Computation Pipeline}
Some works recognize that standard execution results in a sequential bottleneck, whereby GPU-intensive computations are forced to wait for network-intensive data transfers. To eliminate this dependency, the input tensor is partitioned into smaller chunks. This creates a pipelined workflow that allows the expert computation on one chunk to run concurrently with the all-to-all communication of another~\cite{shi2024schemoe,fsmoe,hwang2023tutel,shi2023pipemoe, Klotski,pan2024parm}. They optimize the pipeline using a uniform, random expert placement strategy, which is complementary to our work.

\paragraph{Expert Placement}

The final family optimises the placement of experts on GPUs. These can be categorised as either determining an optimal placement offline or adapting it at runtime. The former leverages offline profiling to pre-compute a fixed expert layout~\cite{go2025moetuneroptimizedmixtureexpert, he2024expertflow}. By analysing historical routing data or cross-layer expert affinities, they create a globally optimised, yet static, placement. 
The latter reacts to observed traffic by moving or creating replicas of popular experts on underutilised GPUs~\cite{zhai2023smartmoe,he2022fastermoe, flexmoe}. Another work combines the expert-centric approach (moving data) with a data-centric alternative (moving experts)~\cite{liu2023janus}. 
A more recent paradigm dynamically rearranges the placement of sequences between devices~\cite{liunetmoe, chen2024luffy}.
However, all these online works are reactive, optimizing placements based on past patterns. This inherent lag will lead to suboptimal performance for dynamic workloads, leaving a critical gap for our proactive design that can anticipate and adapt to incoming demands.

\section{Motivation}
\label{sec:motivation}

This section analyses the characteristics of expert parallelism and identifies the bottleneck of existing expert placement approaches, revealing an opportunity for optimization.

\begin{figure}[t]
    \centering
    \includegraphics[width=0.8\linewidth]{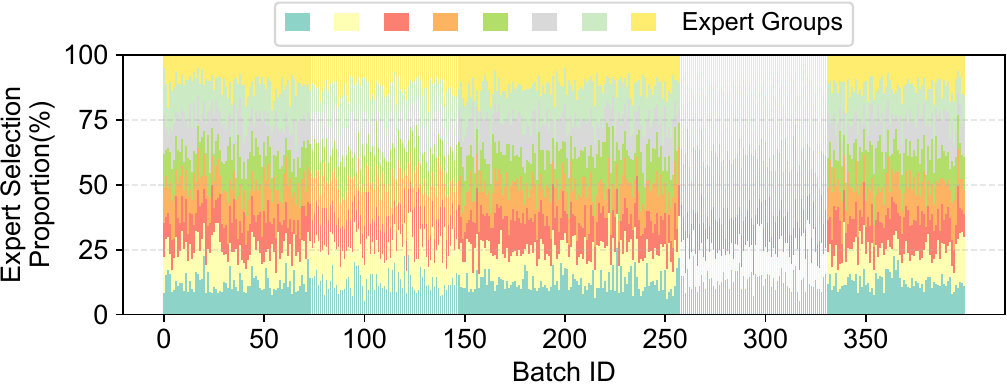}
    \vspace{-0.1cm}
    \caption{Expert selection frequency in a single layer (layer 20 as an example) of DeepSeekMoE 16B with a batch size of $32$.}
    \label{fig:finding2}
    \vspace{-0.5cm}
\end{figure}

\begin{tcolorbox}
\begin{insight}
The activation patterns of experts within an MoE model exhibit significant dynamism when processing different request batches over time.
\end{insight}
\end{tcolorbox}

Both the offline and online reactive approaches operate on the assumption that the patterns of expert activation are relatively stable during the serving process~\cite{zhai2023smartmoe,he2024expertflow, go2025moetuneroptimizedmixtureexpert, he2022fastermoe,flexmoe,zhang2025popfetcher}.
However, this may not be valid in realistic, multi-user serving environments.
To demonstrate this, we randomly select requests from text generation~\cite{wikitext}, mathematical reasoning~\cite {math_dataset}, code generation~\cite{jain2024livecodebench} and record their expert activation patterns.
As shown in \autoref{fig:finding2}, the selection frequency of any given expert varies significantly.
This instability calls into question the effectiveness of both offline and online reactive approaches. This calls for a new \emph{online proactive} approach that optimizes expert placement based on the expert activation patterns of incoming requests rather than the completed ones.

\begin{tcolorbox}
\begin{insight}
When it comes to incoming requests, it is more effective to optimise expert placement based on their activation patterns than on past requests.
\end{insight}
\end{tcolorbox}

\begin{figure}
    \centering
    \includegraphics[width=0.8\linewidth]{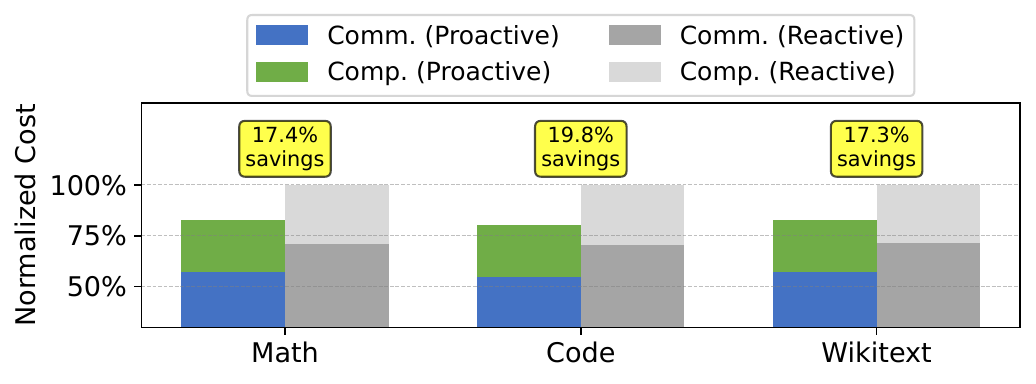}
    \vspace{-0.1cm}
    \caption{The maximum reduction in latency achieved by the design of online proactive expert placement.
    }
    \label{fig:finding3}
    \vspace{-0.5cm}
\end{figure}

Although an online proactive approach does not yet exist, we evaluate its potential benefits below. 
We choose a small-scale MoE model Mixtral 8$\times$7B and compare the optimal latency of reactive and proactive strategies through exhaustive searching.
In the reactive strategy, the optimal placement determined from the previous batch of requests is applied to the next one.
In the proactive strategy, the optimal placement is determined from the incoming batch itself. We limit the number of copies of each expert to one here, since exhaustive searching would be impractical if expert replication were considered.
We adopt the Cluster A configuration in \autoref{sec:evaluation}.
As shown in \autoref{fig:finding3}, the optimal end-to-end latency can be reduced by up to $20\%$.
Furthermore, we believe that the reduction will be amplified when considering expert replication.
This shows the significant potential of online proactive design.

\begin{tcolorbox}
\begin{insight}
Existing predictors of expert activation patterns with negligible overhead perform poorly in the state-of-the-art fine-grained MoE models.
\end{insight}
\end{tcolorbox}

\begin{figure}[t]
	\centering
	\includegraphics[width=0.85\linewidth]{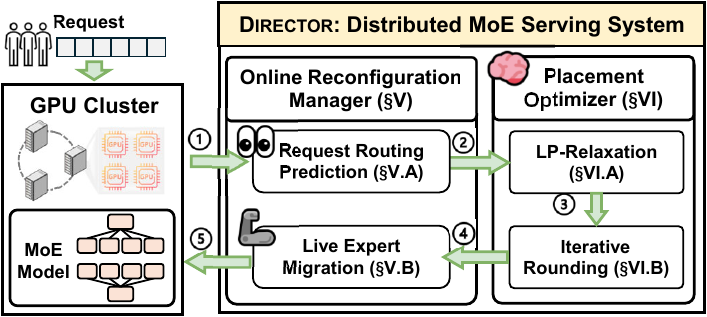} 
	\caption{The workflow of \textsc{Director}.}
	\label{fig:system_overview}
\end{figure}

\begin{table}[t]
\centering
\caption{Performance of the predictor across different MoE models. The overhead is measured as the ratio of prediction time to the per-token end-to-end generation latency.}
\label{tab:finding3}
\begin{tabular}{lcccc}
\toprule
\makecell{Model} &\makecell{Expert Num.\\(per Layer)} & \makecell{Top-$k$\\Routing} & \makecell{Prediction\\Accuracy} & \makecell{Relative\\Overhead} \\
\midrule
Mixtral-8$\times$7B &8      &2      & $61.45\%$     & $1.70\%$  \\
DeepSeekMoE-16B     &64     &6      & $56.14\% $    & $2.71\%$  \\
DeepSeek-V2-Lite    &64     &6      & $55.17\% $    & $2.58\%$  \\
Qwen3-30B-A3B       &128    &8      & $46.69\% $    & $1.59\%$  \\
\bottomrule
\end{tabular}
\vspace{-0.3cm}
\end{table}

Implementing such an online proactive design needs to identify the expert routing pattern of incoming workload. One possible solution would be to use a small neural network for prediction.
However, most existing predictors~\cite{du2024sida,he2024expertflow} are only designed for traditional MoE models, e.g. Switch Transformers~\cite{fedus2022switch}, that select a top-1/2 expert(s) per token.
This problem is exacerbated by the architectural trend toward fine-grained MoE: state-of-the-art models are rapidly scaling their number of experts and top-k routing values.
To evaluate the feasibility, we train a single-layer Transformer with limited overhead budget on the gate network's logits from historical requests—and evaluating it on several state-of-the-art MoE models.
As shown in \autoref{tab:finding3}, although the overhead is negligible, the predictor achieves up to 60\% accuracy. 
Thus, we need to design a new predictor for the online proactive design in the state-of-the-art fine-grained MoE models.

\section{System Overview}
\label{sec:overview}

Motivated by the findings in \autoref{sec:motivation}, we propose \textsc{Director}, a new distributed MoE serving system with three design goals:
\begin{itemize}
    \item \textbf{Proactive Placement}: Rather than optimising for past expert activation patterns, the system should optimise expert placements based on the incoming workloads.
    \item \textbf{Live Reconfiguration}: The reconfiguration of expert placement should result in minimal downtime and a limited impact on performance.
    \item \textbf{Performance Guarantee}: 
    The placement algorithm should be designed to provide a guaranteed performance solution in a limited amount of time.
\end{itemize}

The system overview is shown in \autoref{fig:system_overview}. 
It employs two key components: an online reconfiguration manager (\autoref{design1}) and a relaxation-based placement optimizer (\autoref{design2}). 
The former is responsible for perceiving expert routing of incoming workloads and executing expert reconfiguration.
The latter is responsible for identifying an optimised expert placement strategy based on predicted expert routing.

The system maintains a queue of pending requests and the online reconfiguration manager uses this to predict the workload. Based on this prediction, the placement optimizer computes an updated placement to minimize the expected latency. Once the updated placement has been obtained, the live migration mechanism schedules the required migrations to minimize downtime and limit the impact on performance.

\section{Online Expert Placement Reconfiguration Mechanism}
\label{design1}

As mentioned in \autoref{sec:introduction} and \autoref{sec:motivation}, to maintain an effective placement, the first two challenges are: how to predict the expert routing for pending requests, and how to make expert migrations without interrupting the normal serving process.
Therefore, we introduce an online expert placement reconfiguration mechanism. As illustrated in \autoref{fig:system_overview}, the mechanism operates as follows: first, it employs an adaptive predictor to forecast the expert activation patterns for pending requests in the queue at a low cost. Based on this prediction, it uses the algorithm from \autoref{design2} to get an optimal placement. Finally, it migrates experts during idle periods in GPU communication.

\begin{figure}[t]
	\centering
	\includegraphics[width=0.83\linewidth]{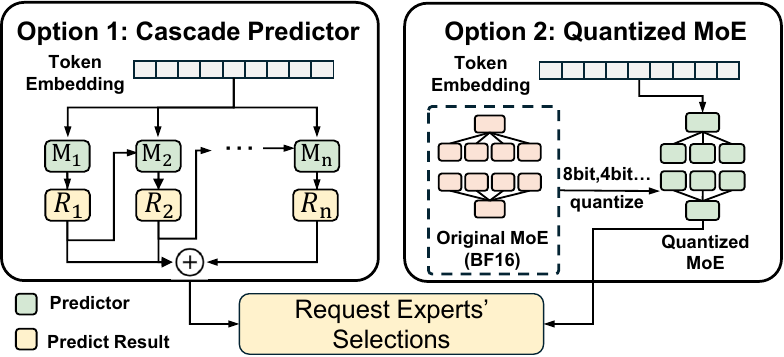} 
	\caption{Two options of predictors in \textsc{Director}}
	\label{fig:predict_mode}
    \vspace{-0.4cm}
\end{figure}

\subsection{Adaptive Routing Predictor}
\label{predictor}

To accommodate various application scenarios, we provide two options for predicting expert activation patterns, as shown in \autoref{fig:predict_mode}: \emph{cascade predictor} and \emph{quantized predictor}.

\paragraph{Cascade Predictor.}
We observe a limitation in the naive predictor in \autoref{sec:motivation} that maps token embeddings to top-k expert choices: its predictions for deeper MoE layers become less accurate. 
This is because the routing in the deeper layers relies on the routing established previously.
A simple predictor that only considers the initial embeddings lacks this conditional context, resulting in the accumulation of prediction errors across layers.
Guided by this observation, we model routing as a conditional sequence of decisions. Accordingly, we adopt a cascade design that explicitly conditions later predictions on earlier ones. This design: (1) injects earlier routing predictions as context to curb error propagation in deeper layers, and (2) offers tunable depth to trade accuracy for overhead.
As shown in \autoref{fig:predict_mode}, we stack single‑layer transformer predictors $\{M_1, M_2, \dots, M_n\}$. 
$M_1$ consumes token embeddings and predicts the expert routing $R_1$; each subsequent $M_i$ augments these embeddings with the predicted routing $R_{i-1}$ from $M_{i-1}$ to predict routing $R_i$ for its designated block of subsequent layers.

\paragraph{Quantized Predictor.}
For scenarios demanding maximum forecast fidelity, our second option is a low-bit quantized replica of the original MoE model. 
With the same model architecture and pre-trained knowledge~\cite{d2moe}, its gating networks generate expert rankings that are highly consistent with the full-precision model, ensuring reliable predictions. 
To manage the cost, we quantize both the replicas' weights and activations (e.g., to INT4). 
Although the model output might be different, we validate that its predictions closely track the original model's routing in \autoref{sec:evaluation}.

These two options offer a practical trade-off between accuracy and overhead: the cascade predictor prioritizes low latency for frequent, minor adjustments, while the quantized predictor ensures high fidelity for critical, major ones.

\subsection{Computation-Overlapped Live Migration}

\begin{figure}[t]
	\centering
	\subfloat[][Naive ``Stop-the-World'']{
		\begin{minipage}[t]{0.8\linewidth}
			\centering
			\includegraphics[width=0.9\linewidth]{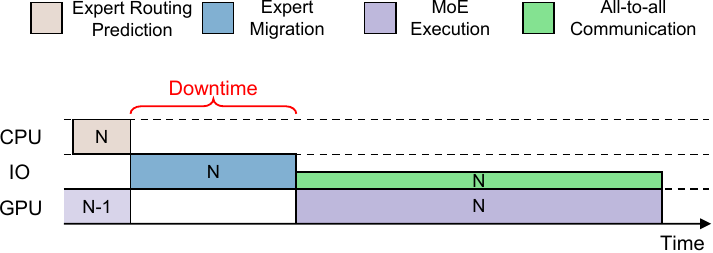 }
		\end{minipage}
        \vspace{-0.2cm}
		\label{fig:design1a}
	}\\
	\subfloat[][Naive Pipeline-based Optimization]{
		\begin{minipage}[t]{0.8\linewidth}
			\centering
			\includegraphics[width=0.9\linewidth]{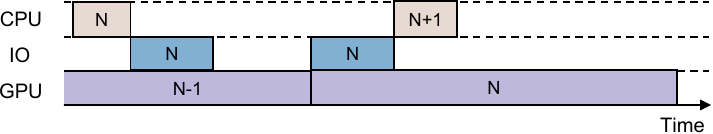 }
		\end{minipage}
        \vspace{-0.2cm}
		\label{fig:design1b}
	}\\
        \subfloat[][Our Pipelining]{
		\begin{minipage}[t]{0.8\linewidth}
			\centering
			\includegraphics[width=0.9\linewidth]{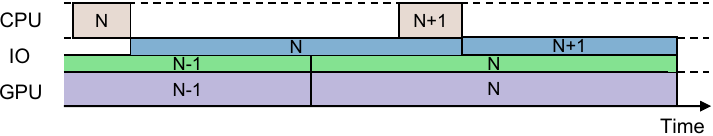 }
		\end{minipage}
        \vspace{-0.2cm}
		\label{fig:design1c}
	}\\
	\centering
	\caption{Execution timelines of serving and reconfiguration.}
    
	\label{fig:design1}
    \vspace{-0.4cm}
\end{figure}

To migrate the experts, A naive ``stop-the-world'' design, illustrated in \autoref{fig:design1a}, is inefficient because it halts all computation during the expert migration.
A naive solution is to pipeline migration with computation.
For example, the system can migrate the first half layers of experts while the experts in subsequent layers are processing tokens in \autoref{fig:design1b}. 
However, this approach overlooks a critical resource conflict: both the expert migration and the expert parallelism's all-to-all operations require communication bandwidth. Scheduling them at the same time can lead to congestion.

We thus design a fine-grained scheduling strategy. 
Its core idea is to prevent expert migration and all-to-all communication from executing concurrently, as shown in \autoref{fig:design1c}. 
As mentioned in \autoref{sec:background}, expert parallelism exists computation phase (attention, FFN) and the communication phase (all-to-all); our strategy schedules expert migrations to run only during computation phases when the communication network is idle. 
By leveraging the expert routing prediction, the module estimates the duration of the computation and communication windows, ensuring the migration completes before the next all-to-all operation begins. 
This approach effectively hides the migration overhead, ensuring the expert placement reconfiguration does not impact the critical path of inference.

\section{Relaxation-Based Expert Placement Algorithm}
\label{design2}

\subsection{System Model}
\label{sec:system_model}

\paragraph{GPU Cluster}
We model a cluster as a set of GPUs $\mathcal{G}$. The maximum number of experts per GPU per layer due to memory limitations is denoted by $C_{\text{max}}$.

\paragraph{Expert Placement}
The MoE model consists of $L$ MoE layers, each having $E$ experts, indexed from $1$ to $E$. 
The placement map $P$ assigns each expert to a host GPU.
For example, $P(l,i)$ denotes the GPU hosting expert $i$ in layer $l$.

\paragraph{Request Workload}
For each request awaiting execution, the predictor in \autoref{design1} predicts the token-to-expert routing pattern.
We use a \emph{workload matrix}, $W$, to indicate the predicted workload.
An element $W_{j,i}^l$ of this matrix indicates the number of tokens on GPU $j$, routed to expert $i$ in layer $l$.

\paragraph{Computation Latency}
The workload on a GPU is dictated by the total workload of the experts it hosts. Given the workload matrix $W^l$ and a placement map $P$, the load on GPU $j$ at layer $l$, denoted $\lambda_{j}^{l}$, is the sum of demands for all experts assigned to it:
\begin{equation}
    \lambda_{j}^{l}(P) = \sum_{i=1}^{E} \mathbb{I}[P(l,i)=j] \left( \sum_{k \in \mathcal{G}} W_{k,i}^l \right),
\end{equation}
where $\mathbb{I}[\cdot]$ is the indicator function. 
The layer's computation latency, $T_{\text{comp}}^{l}$, is determined by the straggler GPU. 
We use the function $f_{\text{comp}}(\lambda)$, derived from offline profiling, to denote the time required to process $\lambda$ tokens. 
The total computation latency is:
\begin{equation}
    T_{\text{comp}}(P) = \sum_{l=1}^{L}\left(\max_{j \in \mathcal{G}} f_{\text{comp}}\left(\lambda_{j}^{l}(P)\right)\right) . 
\end{equation}

\paragraph{Communication Latency}
Let $d$ be the data size for a single token. For a given expert placement $P$, the total data volume to be transferred from GPU $j$ to GPU $k$ at layer $l$ is denoted by $D_{j,k}^l(P)$, calculated as:
\begin{equation}
    D_{j,k}^l(P) = \sum_{i : P(l, i) = k} d\cdot W_{j,i}^l.
\end{equation}

The transmission latency between GPUs is modelled by the function $f_{\text{comm}}( D_{j,k}^l(P))$. It is empirically profiled and is sensitive to the network topology (e.g., intra-node vs. inter-node communication). The layer's overall communication latency is dictated by the straggler, calculated as:
\begin{equation}
    T_{\text{comm}}(P) = \sum_{l=1}^{L}\left(\max_{j \in \mathcal{G}} \sum_{\substack{k \neq j}}f_{\text{comm}}\left( D_{j,k}^l(P) \right)\right).
\end{equation}

\subsection{Problem Formulation}
\label{problem}

Our goal is to find an expert placement that minimizes the total latency of computation and communication: 

\begin{align}
    \mathop{\arg\min}\limits_{P} \quad & T_{\text{comp}}(P) + T_{\text{comm}}(P)\label{eq:p1_obj} \\
    \text{s.t.} \quad & \sum_{i=1}^{E} \mathbb{I}[P(l, i) = j] \leq C_{\text{max}},
     \quad \forall j, \forall l .\notag\label{eq:p1_const}
\end{align}

Solving this problem is challenging, since it involves a trade-off between co-locating experts to save communication and distributing them to balance computation. 
Moreover, the problem is NP-hard, reducible to the Generalized Assignment Problem (GAP)~\cite{savelsbergh1997branch}—and features an immense search space that thwarts general heuristic algorithms. 

\subsection{Algorithm Design}

To solve this challenging problem, our proposed algorithm finds a near-optimal placement in two main phases. First, we transform the problem into a feasibility problem and use a binary search over the total latency budget, $\beta$, to find the minimum achievable latency. At each search step, we solve a Linear Program (LP) relaxation of the problem. Second, after finding the optimal fractional solution $\bar{x}_{\text{best}}$ corresponding to the tightest budget, we employ an iterative rounding scheme to convert it into a final integer expert placement $P$. The overall process is outlined in \autoref{alg:main}.

\paragraph{LP Relaxation with Linearization}
We define a binary variable $x_{l,i,j} \in \{0, 1\}$ to indicate if expert $i$ in layer $l$ is placed on GPU $j$. To create a tractable LP, we first express the computational load and communication data volume using these variables. Let $\lambda_{j}^{l}(x) = \sum_{i=1}^{E} x_{l,i,j} ( \sum_{k \in \mathcal{G}} W_{k,i}^l )$ be the expected computational load on GPU $j$ at layer $l$. Similarly, let $D_{j,k}^l(x) = d \cdot \sum_{i=1}^{E} W_{j,i}^l x_{l,i,k}$ be the expected data volume transferred from GPU $j$ to $k$.

We then linearize the overall non-linear objective function by introducing auxiliary variables for the \textit{max} operators and employing a standard piecewise linear approximation for the latency functions $f_{\text{comp}}$ and $f_{\text{comm}}$. For a given latency budget $\beta$, the resulting feasibility LP is defined as:
\begin{subequations}
\label{problem:lp_feasibility}
\begin{align}
    \text{find} \quad & x \\
    \text{s.t.} \quad & \sum_{l=1}^L (\tau^l_{\text{comp}} + \tau^l_{\text{comm}}) \le \beta \label{eq:lp_budget}\\
    & \tau^l_{\text{comp}} \ge f_{\text{comp}}\left( \lambda_{j}^{l}(x) \right), && \forall l, j \label{eq:lp_comp}\\
    & \tau^l_{\text{comm}} \ge \sum_{k \neq j} f_{\text{comm}}\left( D_{j,k}^l(x) \right), && \forall l, j \label{eq:lp_comm}\\
    & \sum_{j \in \mathcal{G}} x_{l,i,j} \ge 1, &&\forall l, i \label{eq:lp_assign}\\
    & \sum_{i=1}^{E} x_{l,i,j} \le C_{\text{max}}, &&\forall l, j \label{eq:lp_capacity}\\
    & 0 \le x_{l,i,j} \le 1, &&\forall l, i, j \label{eq:lp_vars}
\end{align}
\end{subequations}
where $\tau^l_{\text{comp}}$ and $\tau^l_{\text{comm}}$ are auxiliary variables representing the computation and communication latency at layer $l$.

\paragraph{Iterative Rounding}
After obtaining the optimal fractional solution $\bar{x}_{\text{best}}$, we adapt an iterative rounding design to derive a final integer placement. The procedure is a randomized algorithm that performs a constrained ``random walk'' in the solution space. As proved in \S\ref{sec:analysis}, this process guarantees a near-optimal integral solution in polynomial time.

The rounding process consists of $O(\log N_{\text{vars}})$ (see \autoref{sec:analysis}) macro-iterations. In each, we further reduce the number of fractional variables. The core mechanisms, which are essential for the subsequent analysis, are detailed below.

\textit{1. Initialization and Variable Scaling:}
At the start of a macro-iteration, we classify each fractional variable $x_v$ (where $v$ is a multi-index for $(l,i,j)$) based on its value. For each variable, we define a discrete \emph{scale} $s_v=2^{-k}$ if its value $x_v$ or $1-x_v$ falls into the range $(2^{-(k+1)}, 2^{-k}]$. This scale is crucial as it dictates the magnitude of perturbation the variable will receive; variables closer to an integer value have a smaller scale and are moved more cautiously.

\begin{algorithm}[t]
\SetAlgoLined
\caption{Relaxation-based Placement Algorithm}
\label{alg:main}
\KwIn{Workload matrices $\{W^1,W^2,\dots,W^L\}$, GPU set $\mathcal{G}$, capacity $C_{\text{max}}$}
\KwOut{Optimal expert placement map $P$}

\BlankLine
$\text{Initialize the budget bound }(\beta_{\text{low}}, \beta_{\text{high}})$\;
$\bar{x}_{\text{best}} \gets \text{null}$\;
\While{$\beta_\text{high} $ - $ \beta_\text{low}> \epsilon$}{
    $\beta \gets (\beta_{\text{low}} + \beta_{\text{high}}) / 2$\;
    $\text{solve LP relaxation with given } \beta \rightarrow \bar{x}$\;
    $\text{update search bounds }(\beta_\text{high}, \beta_\text{low})$\;
}
$\text{get integer placement P }\gets \bar{x}_{\text{best}}\text{ by iterative rounding}$\;
\Return{$P$}\;

\end{algorithm}
\vspace{-0.4cm}

\textit{2. Constrained Random Walk:}
The random walk is performed within a carefully defined polytope $Q$ that preserves feasibility. This polytope is defined by three sets of constraints: (i) the original LP constraints (Eq. \ref{problem:lp_feasibility}), (ii) scale-based variable bounds of the form $x_v \le \alpha \cdot s_v$ (and $1-x_v \le \alpha \cdot s_v$), where $\alpha$ is a constant, and (iii) an additional set of \emph{scale-preserving constraints}. These take the form:
\begin{equation}
    \sum_{v \in U_k} x_v = \sum_{v \in U_k} x_v^{(0)}, \quad \forall k
\end{equation}
where $U_k$ is the set of variables with the same scale $s_v=2^{-k}$, and $x^{(0)}$ is the solution at the beginning of the macro-iteration. These constraints ensure that while individual variables are perturbed, their collective sum within a scale group remains constant, thus forcing some variables towards the integer boundaries of 0 or 1.

The walk proceeds for $T_{\text{micro}}$ micro-steps, where $T_{\text{micro}}$ is a sufficiently large polynomial in $N_{\text{vars}}$. At each step $t$, the update direction $\mathbf{g}^{(t)}$ is sampled from the null space of the tight constraints of $Q$. This direction is scaled such that variables with smaller scales are perturbed less. The solution is then updated by a small, fixed step size $\gamma$:
\begin{equation}
    x^{(t+1)} \leftarrow x^{(t)} + \gamma \cdot \mathbf{g}^{(t)}
\end{equation}

\textit{3. Integrality Guarantee and Termination:}
If any variable $x_v$ exits the $[0,1]$ range during the walk, it is permanently ``clamped'' to 0 or 1. This is the primary mechanism that reduces the number of fractional variables. The convergence of this process is theoretically guaranteed by a \emph{potential function} $\Phi(x)$, defined as:
\begin{equation}
    \Phi(x) = \sum_{v \text{ is fractional}} s_v^{-2} \cdot \text{dist}(x_v, \{0,1\})^2
\end{equation}
where $\text{dist}(x_v, \{0,1\})$ is the distance of $x_v$ to its nearest integer. The algorithm is designed to ensure that the expected value of $\Phi(x)$ strictly increases with each random step. This provable progress guarantees that a constant fraction of variables become integral in each macro-iteration. After $T_{\text{micro}}$ steps, the macro-iteration concludes.

\section{Analysis}
\label{sec:analysis}

In this section, we prove the approximation guarantee of \autoref{alg:main}. 
To facilitate the following analysis, let $N_{\text{vars}} = L \cdot E \cdot |\mathcal{G}|$ denote the total number of decision variables.

\begin{theorem}
\label{thm:single_step}
    A single macro-iteration of the \emph{Iterative Rounding} scheme is a randomized polynomial-time procedure. It transforms a fractional solution $x$ into a new fractional solution $x'$ that simultaneously satisfies:
    \begin{enumerate}
        \item \textbf{Integrality Progress:} The number of fractional variables is reduced by a constant factor with high probability.
        \item \textbf{Structural Preservation:} The new solution $x'$ remains within the feasible placement polytope $Q$.
        \item \textbf{Bounded Latency Violation:} The violation of any single linear latency constraint is controllably small and dependent on the step size $\gamma$.
    \end{enumerate}
\end{theorem}

\begin{lemma}
\label{lem:preservation}
    The rounding macro-iteration maintains feasibility with respect to all constraints defining the polytope $Q$.
\end{lemma}
\begin{proof}
    This property is guaranteed by construction. As defined in the algorithm description, the random walk is confined within the polytope $Q$. The update direction $\mathbf{g}^{(t)}$ is drawn from a subspace orthogonal to all tight constraints, ensuring the solution does not violate them. Any "clamping" of variables is a projection onto the boundary of $Q$. Thus, the updated solution $x'$ remains feasible.
\end{proof}

We define the \textit{scaled variance} of a linear constraint $\Lambda(x) = \mathbf{c}^T x$ as $\text{Var}_s(\Lambda) = \sum_{v} c_v^2 s_v^2$, where the sum is over the set of fractional variables.

\begin{lemma}
\label{lem:violation}
    For any linear latency constraint $\Lambda(x) = \mathbf{c}^T x$, the total violation $|\Lambda(x') - \Lambda(x)|$ in a single macro-iteration is bounded by $C_1 \cdot \gamma \cdot \sqrt{\text{Var}_s(\Lambda)}$ with high probability, for some constant $C_1$.
\end{lemma}
\begin{proof}
    The total change, $\Delta\Lambda = \sum_{t=0}^{T_{\text{micro}}-1} Z_t$, is a sum of martingale differences $Z_t = \gamma \mathbf{c}^T \mathbf{g}^{(t)}$. The conditional variance of each step is bounded by $\mathbb{E}[Z_t^2 | \text{history}] \le \gamma^2 \text{Var}_s(\Lambda)$. The total conditional variance over $T_{\text{micro}}$ steps is thus $W \le T_{\text{micro}} \gamma^2 \text{Var}_s(\Lambda)$. By applying martingale concentration inequalities, such as Freedman's Theorem~\cite{fre75}—a tool also central to the analysis in~\cite{discrepancy_rounding}—the probability of the total deviation $|\Delta\Lambda|$ exceeding a threshold proportional to $\gamma \sqrt{W / \gamma^2} = \gamma \sqrt{T_{\text{micro}} \text{Var}_s(\Lambda)}$ is exponentially small. A union bound over all latency constraints completes the argument, with $C_1$ absorbing terms like $\sqrt{T_{\text{micro}}}$.
\end{proof}

\begin{claim}
\label{claim:bounded_variance}
    For any linear latency constraint $\Lambda$ and a solution with $f$ fractional variables, the scaled variance is bounded: $\text{Var}_s(\Lambda) \le C_2 \cdot f$, for a constant $C_2$.
\end{claim}
\begin{proof}
    The scaled variance is a sum over the $f$ fractional variables: $\text{Var}_s(\Lambda) = \sum_{v \in F} c_v^2 s_v^2$. The coefficients $c_v$ are derived from the workload matrices and are bounded by a maximum value $c_{\max}$. The scales $s_v$ are, by definition, at most $1/2$, so $s_v^2 \le 1/4$. Thus, $\text{Var}_s(\Lambda) \le \sum_{v \in F} c_{\max}^2 \cdot (1/4) = f \cdot (c_{\max}^2/4)$. Let $C_2 = c_{\max}^2/4$.
\end{proof}

\begin{lemma}[]
\label{lem:progress}
    Let $F$ and $F'$ be the sets of fractional variables before and after a single macro-iteration. With high probability, $|F'| \le (1-\delta)|F|$ for some constant $\delta > 0$.
\end{lemma}
\begin{proof}
    The proof uses a potential function argument, similar to~\cite{ban10, lm12}. The total expected increase in the potential function $\Phi(x)$ over a macro-iteration is lower-bounded: $\mathbb{E}[\Phi(x') - \Phi(x)] \ge T_{\text{micro}} \gamma^2 \mathbb{E}[\dim(\mathcal{V}')] - C_{\text{trunc}}$, where $C_{\text{trunc}}$ is a small constant accounting for truncation effects.
    
    The potential function $\Phi(x')$ is bounded above by $O(N_{\text{vars}})$ due to the scale-based constraints $x_v \le \alpha \cdot s_v$ within the polytope $Q$. This implies that the average walk space dimension, $\mathbb{E}[\dim(\mathcal{V}')] = \mathbb{E}[|F| - |\mathcal{C}_{\text{tight}}|]$, must be small, which forces the expected number of tight constraints $\mathbb{E}[|\mathcal{C}_{\text{tight}}|]$ to be large.

    Crucially, these tight constraints must correspond to variables becoming integral. The number of tight latency constraints is shown to be small by applying Lemma \ref{lem:violation}. The number of variables clamped to non-integral boundaries is limited by the scale-preserving constraints. As detailed in~\cite{discrepancy_rounding}, these constraints imply that only a small fraction of variables can be clamped at non-integral boundaries. Therefore, the majority of tight constraints must originate from variables being clamped to $\{0, 1\}$, forcing a constant fraction of fractional variables to become integral.
\end{proof}

\begin{proof}[Proof of Theorem \ref{thm:single_step}]
    The theorem's three properties follow directly from Lemmas \ref{lem:preservation}, \ref{lem:violation}, and \ref{lem:progress}. The polynomial runtime follows from its construction from polynomially-solvable linear algebra operations.
\end{proof}

\begin{theorem}
\label{thm:approximation_guarantee}
    For any valid input, \autoref{alg:main} is a randomized polynomial-time procedure that, for any $\epsilon > 0$, returns an integer placement $P$ with total latency $T(P) \le (1+\epsilon)\beta^\star$ with high probability, where $\beta^\star$ is the optimal LP latency.
\end{theorem}
\begin{proof}
    We begin by solving the LP with a budget of $\beta' = (1 + \epsilon/2)\beta^\star$, creating an error slack of $(\epsilon/2)\beta^\star$. The rounding process is then executed for $T = O(\log N_{\text{vars}})$ macro-iterations. The final latency is $T(P) \le \beta' + \sum_{t=1}^{T} \Delta_t$, where $\Delta_t$ is the latency violation from macro-iteration $t$.

    We now bound the total accumulated violation. From Lemma \ref{lem:violation} and Claim \ref{claim:bounded_variance}, the violation in a single step $t$ (starting with $f_{t-1}$ fractional variables) is bounded by:
    $$ \Delta_t \le C_1 \cdot \gamma \cdot \sqrt{\text{Var}_s(\Lambda_t)} \le C_1 \cdot \gamma \cdot \sqrt{C_2 \cdot f_{t-1}} $$
    Let $C_3 = C_1\sqrt{C_2}$. From Lemma \ref{lem:progress}, we know that the number of fractional variables decreases geometrically: $f_t \le (1-\delta)f_{t-1}$. Let $f_0 \le N_{\text{vars}}$ be the initial number of fractional variables. The total violation is a sum over a converging geometric series:
    \begin{equation}
        \label{eq:error_sum}
        \begin{aligned}
            \sum_{t=1}^{T} \Delta_t &\le \sum_{t=1}^{T} C_3 \cdot \gamma \cdot \sqrt{f_{t-1}} \\
            &\le C_3 \cdot \gamma \cdot \sqrt{f_0} \sum_{k=0}^{T-1} (\sqrt{1-\delta})^{k} \\
            &\le C_3 \cdot \gamma \cdot \sqrt{N_{\text{vars}}} \cdot \frac{1}{1 - \sqrt{1-\delta}}
        \end{aligned}
    \end{equation}
    Let $C_4 = C_3 \cdot (1 - \sqrt{1-\delta})^{-1}$. The total error is bounded by $C_4 \cdot \gamma \cdot \sqrt{N_{\text{vars}}}$. To ensure this error is bounded by the allocated slack, we can select the step size $\gamma$ such that:
    $$ \gamma \le \frac{(\epsilon/2)\beta^\star}{C_4 \sqrt{N_{\text{vars}}}} $$
    This choice guarantees that the total added violation is no more than $(\epsilon/2)\beta^\star$. This trade-off between error magnitude and runtime (a smaller $\gamma$ may require a larger polynomial $T_{\text{micro}}$) is a standard component of such rounding frameworks~\cite{lrs11}.

    Consequently, the final latency is bounded by $T(P) \le (1+\epsilon/2)\beta^\star + (\epsilon/2)\beta^\star = (1+\epsilon)\beta^\star$.
\end{proof}

\begin{theorem}
\label{thm:time_complexity}
\autoref{alg:main} runs in polynomial time.
\end{theorem}
\begin{proof}
The algorithm consists of two phases. The binary search phase requires only a logarithmic number of calls to a standard polynomial-time LP solver. The subsequent iterative rounding phase is also fast, consisting of a logarithmic number of macro-iterations, each involving only efficient linear algebra operations.
Since both sequential phases run in polynomial time, the time complexity of \autoref{alg:main} is polynomial.
\end{proof}

\section{Evaluation}
\label{sec:evaluation}

\begin{table}[t]
\centering

\begin{tabular}{lcccc}
\toprule
\multirow{2}{*}{MoE Model} &
\multicolumn{2}{c}{Cascade} &
\multicolumn{2}{c}{Quantization} \\
\cmidrule(lr){2-3} \cmidrule(lr){4-5}
& Acc.(\%) & Size (\%) & Acc.(\%) & Size (\%) \\
\midrule
Mixtral‑8$\times$7B   & 91.44 & 0.09 & 94.10 & 27.36 \\
DeepSeekMoE‑16B       & 84.81 & 0.07 & 90.98 & 29.51 \\
DeepSeek‑V2‑Lite      & 84.62 & 0.08 & 95.68 & 35.62 \\
Qwen3‑30B‑A3B         & 76.99 & 0.04 & 94.36 & 28.11 \\
\bottomrule
\end{tabular}
\caption{Expert routing prediction accuracy and predictor size relative to the served MoE for the two options}
\vspace{-0.4cm}
\label{tab:predictor_performance}
\end{table}

\subsection{Implementation}
We implement our \textsc{Director} prototype in Python on top of Colossal-AI~\cite{Colossal}. The relaxation-based placement optimizer solves LP using SciPy~\cite{2020SciPy-NMeth}, and we use Activation-aware Weight Quantization (AWQ)~\cite{awq} for 4-bit quantization. 
Besides, the PyTorch version is 2.3, and the CUDA version is 12.4.
The communication backend is based on NCCL.

\subsection{Experimental Setup}

\textbf{Testbed.}
We evaluate \textsc{Director} on two representative clusters with  heterogeneous hardware configurations:  
(1) Cluster A: 4 nodes, each with 8 NVIDIA RTX 4090 GPUs; GPUs are connected by PCIe 4.0 within a node and by InfiniBand across nodes.
(2) Cluster B: 4 nodes, each with 8 NVIDIA H200 GPUs; GPUs are connected by NVLINK within a node and by InfiniBand across nodes.

\textbf{Models and Datasets.}
We evaluate \textsc{Director} on four representative MoE models: Mistral 8$\times$7B~\cite{jiang2024mixtral}, DeepSeekMoE‑16B~\cite{deepseek_moe}, DeepSeek‑V2‑Lite~\cite{deepseek_v2}, and Qwen3‑30B‑A3B~\cite{qwen3technicalreport}. 
Selected models cover a wide spectrum of configurations: the number of experts per layer varies from 8 to 128, and the top-k routing from 2 to 8 (details in \autoref{tab:finding3}). 
To emulate a realistic multi-task scenario, we sample prompts from three widely used datasets: MATH500 \cite{math_dataset}, WikiText-103 \cite{wikitext}, and Live Code Bench \cite{jain2024livecodebench}.

\textbf{Baselines.}
We compare \textsc{Director} with following baselines:
(1)Vanilla: the default expert‑placement strategy in DeepSpeed‑MoE~\cite{rajbhandari2022deepspeed}.
(2) Offline (static) placement: an optimal placement is computed once from historical traces and kept fixed during serving, as in ~\cite{go2025moetuneroptimizedmixtureexpert}.
(3) Online reactive placement: the system observes expert‑routing patterns of the current and recent batches and adjusts placement accordingly at runtime, as in ~\cite{zhai2023smartmoe,he2022fastermoe, flexmoe}.

\textbf{Metrics.} 
We evaluate \textsc{Director} and all baselines using three metrics, reporting the mean over five random seeds for each experiment.  
(1) \emph{End‑to‑end latency}: the wall‑clock time from a request’s arrival in the queue to its completion, capturing both computation and any migration overhead.  
(2) \emph{Throughput}: the number of tokens served per second under steady load, reflecting aggregate system efficiency.
(3) \emph{Prediction accuracy}: the top‑$k$ match rate between the predictor’s forecasted experts and the ground‑truth routing.

\begin{figure}[t]
        \centering
	\subfloat[][RTX 4090 Cluster]{
		\begin{minipage}[t]{0.9\linewidth}
			\centering
			\includegraphics[width=\linewidth]{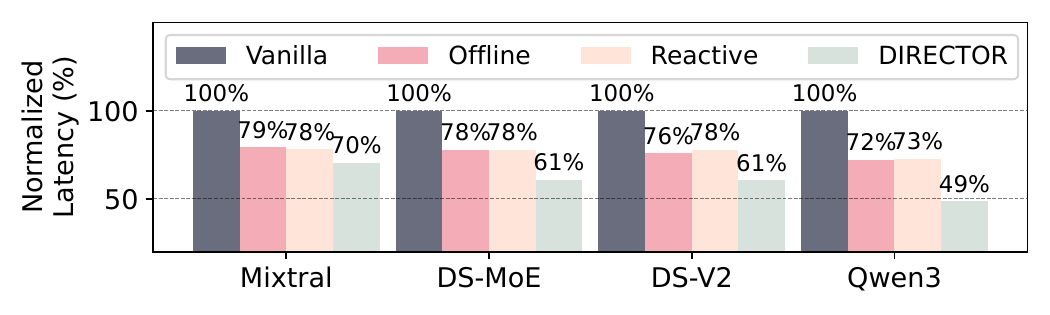} 
		\end{minipage}%
		\label{fig:latency_32_4090}
	}\vspace{-0.2cm}
    \\
	\subfloat[][H200 Cluster]{
		\begin{minipage}[t]{0.9\linewidth}
			\centering
			\includegraphics[width=\linewidth]{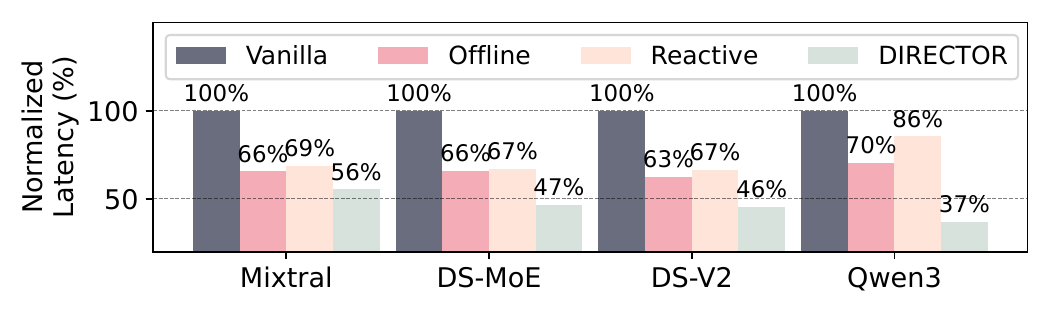} 
		\end{minipage}%
		\label{fig:latency_32_h200}
	}\vspace{-0.2cm}
	\caption{End-to-end normalized latency of \textsc{Director} against baselines on different models and clusters.}
    \label{fig:latency}
    \vspace{-0.4cm}
\end{figure}


\subsection{Results}  


Table~\ref{tab:predictor_performance} summarizes the accuracy and footprint of our two predictor modes.
The cascade predictor (two-stage) achieves $77\sim91\%$ accuracy with a minimal footprint ($<0.1\%$).
Accuracy peaks on Mixtral (8 experts) but declines on complex routing tasks, where a deeper cascade (\autoref{predictor}) can maintain performance.
Alternatively, the quantized predictor (4-bit BF16 replica) reaches $91\sim96\%$ accuracy, nearly matching full-precision routing.
These results reveal a trade-off: the cascade predictor offers high efficiency, while the quantized predictor ensures maximum fidelity.

\autoref{fig:latency} compares the end-to-end latency of \textsc{Director} against baselines.
Our approach consistently achieves the lowest latency, yielding $11\%\sim55\%$ gains over the offline baseline.
This advantage grows with routing complexity; as experts increase (from 8 in Mixtral to 128 in Qwen3), load imbalance risks—resolved by our proactive placement—worsen.
These gains are amplified on the H200 cluster, where high bandwidth reduces communication penalties, allowing \textsc{Director} to focus on balancing computational load.
We also note that baselines are limited by their reactive nature, though the offline approach generally outperforms the myopic reactive strategy due to its longer profiling window.

To analyze robustness, \autoref{fig:latency_batch} shows the per-batch latency trace.
All methods show similar fluctuations, suggesting intrinsic demand variations.
However, the reactive strategy exhibits sharper spikes than offline, as its myopic focus makes it vulnerable to transient extremes.
Conversely, \textsc{Director} demonstrates superior stability, achieving the lowest latency and variance, which indicates robust performance under dynamic loads.

\begin{figure}[t]
	\centering
	\includegraphics[width=0.83\linewidth]{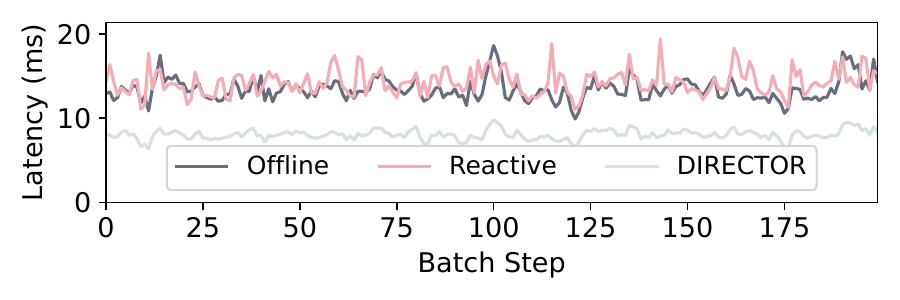}
    \vspace{-0.2cm}
	\caption{The batch‑level latency traces under different expert placement approaches on H200 cluster. }\vspace{-0.2cm}
	\label{fig:latency_batch}
    \vspace{-0.4cm}
\end{figure}

\begin{figure}[t]
	\centering
	\includegraphics[width=0.9\linewidth]{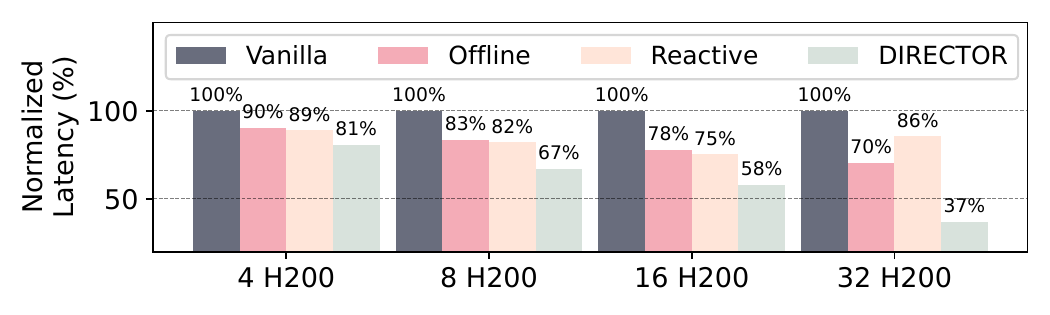}
        \vspace{-0.2cm}
	\caption{End-to-end normalized latency of \textsc{Director} against baselines on different number of GPUs. }
	\label{fig:gpu_scaling}
        \vspace{-0.4cm}
\end{figure}

\begin{figure}[t]
	\centering
	\includegraphics[width=0.9\linewidth]{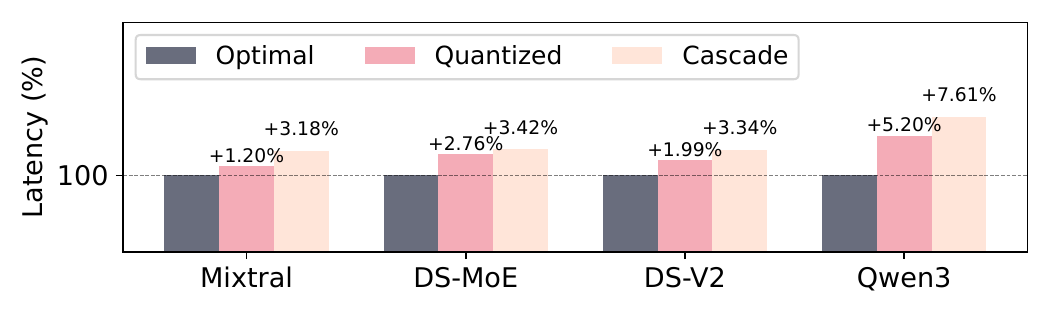}
        \vspace{-0.2cm}
	\caption{Performance under different predictor options.}
	\label{fig:ablation}
        \vspace{-0.4cm}
\end{figure}

\begin{figure}[t]
	\centering
	\includegraphics[width=0.9\linewidth]{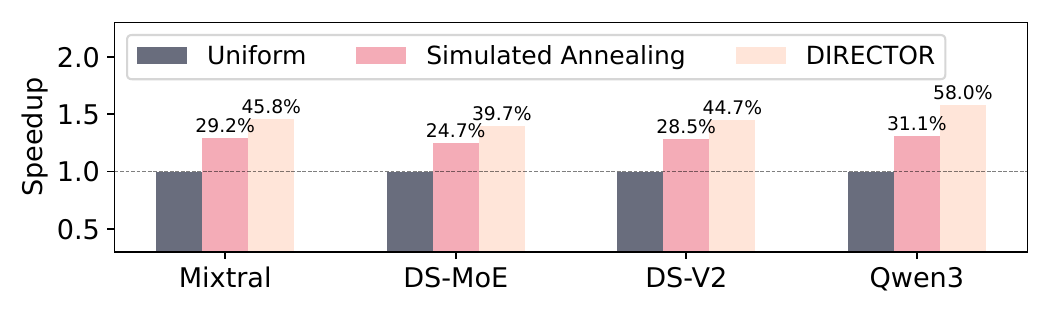}
    \vspace{-0.2cm}
	\caption{Performance under various optimization algorithms.}
	\label{fig:algorithm_compare}
    \vspace{-0.4cm}
\end{figure}


We further evaluate scalability by measuring latency as H200 GPUs increase from 4 to 32.
As shown in \autoref{fig:gpu_scaling}, \textsc{Director}'s advantage becomes more pronounced at scale.
Latency reduction over the offline baseline grows from 11\% (4 GPUs) to 48\% (32 GPUs).
This highlights a key insight: larger systems are more sensitive to load imbalance.
A single mismanaged expert can create a straggler that bottlenecks more devices, amplifying suboptimal placement costs.
\textsc{Director}'s proactive approach effectively mitigates this, whereas baseline efficacy diminishes as the system scales.

\subsection{Ablation Study}

To quantify the performance impact of predictor inaccuracy and validate the effectiveness of our relaxation-based placement algorithm, we conduct the following two ablation studies.

First, we evaluate how predictor fidelity affects overall performance. 
We compare our two predictor options from \autoref{predictor} against an oracle with ground-truth routing on a 32-GPU RTX 4090 cluster, with results summarized in \autoref{fig:ablation}. 
Across the four tested models, the performance gap between our predictors and the oracle is modest, ranging from 1\% to 8\%. 
The cascade predictor results in a performance gap of 3\% to 7\%, while the gap from the quantized predictor is almost negligible at 1\% to 6\%. 
This small performance gap demonstrates that our predictors achieve sufficient accuracy and do not introduce a significant performance gap.

Second, we evaluate our relaxation-based placement algorithm against a no-prediction baseline and a Simulated Annealing (SA) heuristic. 
SA is a common heuristic for large optimization spaces, and we ensure it runs for sufficient iterations to guarantee a fair comparison. 
As shown in \autoref{fig:algorithm_compare}, our algorithm achieves a substantial 39\% to 58\% performance improvement over the baseline. 
In contrast, the SA heuristic, despite an extensive search, yields a more modest gain of 24\% to 32\%. 
This result demonstrates that our relaxation-based approach consistently finds higher-quality solutions than a well-established heuristic, highlighting its effectiveness in navigating the complex placement landscape.

\section{Conclusion}
This paper proposes \textsc{Director}, a distributed MoE serving system designed to overcome the performance limitations of existing offline and reactive placement strategies under dynamic workloads. 
\textsc{Director} materializes this proactive paradigm through its key technical contributions: it employs adaptive predictors to predict routing from pending requests; utilizes a relaxation-based optimizer to find a near-optimal placement in polynomial time with a provable approximation guarantee; and enacts reconfigurations via a computation-overlapped migration module to ensure near-zero downtime. 
Through extensive evaluations, including targeted ablation studies that validate each design component, our prototype demonstrates that this proactive approach reduces end-to-end latency by a significant $11-55\%$ compared to existing systems. 
In future work, we plan to generalize our proactive placement paradigm beyond expert parallelism, applying its principles to other parallel dimensions.

\bibliographystyle{IEEEtran}
\bibliography{ref}

\end{document}